\documentclass[sigconf]{acmart}

\usepackage{booktabs} % For formal tables
\usepackage{amsmath}
\usepackage{amssymb}
\usepackage{amsthm}
\usepackage{geometry}
\usepackage{indentfirst}
\usepackage{graphicx}
\usepackage{multirow}
\usepackage{makecell}
\usepackage{mathrsfs}
\usepackage{bbm}
\usepackage[linesnumbered,ruled]{algorithm2e}
\allowdisplaybreaks

\setcopyright{rightsretained}
\acmDOI{10.475/123_4}

\acmISBN{123-4567-24-567/08/06}

\DeclareMathOperator*{\argmax}{arg\,max}

\newcommand{\sumi}{\sum\limits_i}
\newcommand{\sumj}{\sum\limits_j}
\newcommand{\sumk}{\sum\limits_k}
\newcommand{\sumij}{\sum\limits_{ij}}

\newcommand{\sx}{x_{ij}}
\newcommand{\sbp}{bp_{ij}}

\newcommand{\sumijx}[1]{\sumij\sx{}#1}

\newcommand{\sProb}{Prob_i(\sbp)}
\newcommand{\sCost}{Cost_i(\sbp)}

\newcommand{\sV}{V_{ij}}
\newcommand{\sW}{W_{ij}^{(k)}}
\newcommand{\sB}{B^{(k)}}

\newcommand{\sBudget}{Budget^{(k)}}
\newcommand{\sROI}{ROI^{(k)}}

\newcommand{\inRange}[1]{\in\{1,2,...,#1\}}

\newcommand{\sCPP}{CPP_j}
\newcommand{\sCR}{CR_j}
\newcommand{\sPPI}{PPI_{ij}}
\newcommand{\sCPI}{\sCPP\sPPI}

\newcommand{\sRevenuePforP}{\sumijx{\sCPI\sProb}}
\newcommand{\sRevenuePforU}{\sumijx{(1+\sCR)\sCost}}
\newcommand{\sPerformance}{\sumijx{\sPPI\sProb}}
\newcommand{\sBiddingCost}{\sumijx{\sCost}}

\newcommand{\salpha}{\alpha^{(k)}}
\newcommand{\szeta}{\zeta^{(k)}}
\newcommand{\sbeta}{\beta_i}
\newcommand{\seta}{\eta_i}
\newcommand{\sgamma}{\gamma_{ij}}

\newcommand{\sF}{F_{ij}}
\newcommand{\sS}{S_{ij}}
\newcommand{\sG}{G_i}

\newcommand{\valpha}{\vec{\alpha}}

\newcommand{\pprob}{\phi}
\newcommand{\pcost}{\psi}
\newcommand{\uff}{\mathscr{F}}
\newcommand{\uf}{f(bp; \pprob, \pcost, p(x))}

\newcommand{\dspresourceconstraint}{\sumij \sx \sW(\sbp) \le \sB}
\newcommand{\ammkpresourceconstraint}{\sumij \sx \sW(\sV) \le \sB}

\newcommand{\assignmentconstraint}{\sumj \sx \le 1}

\newcommand{\scoreconstraint}{\sbeta \ge \sS(\vec{\alpha})}

\newcommand{\ortbbp}{\sqrt{\frac{c\sCPI}{ROI}(1+\frac{1}{\lambda})+c^2}-c}
\newcommand{\dbbp}{\frac{\sCPI}{ROI}(1+\frac{1}{\alpha})}

\newcommand{\liniter}{\sbp^{'}=\frac{ActualROI_i(\sbp)}{ROI}Bid}
\newcommand{\ortbiter}{\lambda^{'}=\frac{ROI}{ActualROI(\lambda)}\lambda}
\newcommand{\dbiter}{\alpha^{'} = \frac{ROI}{ActualROI(\alpha)}\alpha}

\newcommand{\mr}[2]{\multirow{#1}{*}{#2}}
\newcommand{\mc}[2]{\multicolumn{#1}{c|}{#2}}

\begin{document}

\title{Dual Based DSP Bidding Strategy and its Application}

\author{Huahui Liu}
\affiliation{%
  \institution{Alibaba Group}
}
\email{huahui.lhh@alibaba-inc.com}

\author{Mingrui Zhu}
\affiliation{%
  \institution{Alibaba Group}
}
\email{mingrui.zmr@alibaba-inc.com}

\author{Xiaonan Meng}
\affiliation{%
  \institution{Alibaba Group}
}
\email{xiaonan.mengxn@alibaba-inc.com}

\author{Yi Hu}
\affiliation{%
  \institution{Alibaba Group}
}
\email{erwin.huy@alibaba-inc.com}

\author{Hao Wang}
\affiliation{%
  \institution{Alibaba Group}
}
\email{longran.wh@alibaba-inc.com}

\begin{abstract}
In recent years, RTB(Real Time Bidding) becomes a popular online advertisement trading method.
During the auction, each DSP(Demand Side Platform) is supposed to
    evaluate current opportunity and respond with an ad and corresponding bid price.
It's essential for DSP to find an optimal ad selection and bid price determination strategy
    which maximizes revenue or performance under budget and ROI(Return On Investment) constraints
    in P4P(Pay For Performance) or P4U(Pay For Usage) mode.
We solve this problem by
    1) formalizing the DSP problem as a constrained optimization problem,
    2) proposing the augmented MMKP(Multi-choice Multi-dimensional Knapsack Problem) with general solution,
    3) and demonstrating the DSP problem is a special case of the augmented MMKP and deriving specialized strategy.
Our strategy is verified through simulation and outperforms state-of-the-art strategies in real application.
To the best of our knowledge, our solution is the first dual based DSP bidding framework
    that is derived from strict second price auction assumption and
    generally applicable to the multiple ads scenario with various objectives and constraints.
\end{abstract}


\maketitle

\section{Introduction} \label{Introduction}

In recent years, RTB(Real Time Bidding) becomes a popular online advertisement trading method.
There are three major roles in the market, namely SSP(Supply Side Platform), DSP(Demand Side Platform), and AdX(Ad Exchange).
SSP controls huge amount of websites and earns money by supplying impressions.
DSP holds a lot of advertisers and makes profit through fulfilling their demands.
AdX, an online advertisement exchange, docks SSPs and DSPs and holds auctions.

In a typical scenario, an audience visits one of the SSP's websites, then the AdX is informed and an auction is initiated.
The AdX broadcasts bid request to DSPs and waits for a short time(e.g. 100ms).
Each DSP is supposed to evaluate current opportunity and respond with an ad and corresponding bid price.
The AdX gathers bid responses arriving before deadline and determines the winner and its bidding cost.
Finally, the AdX notifies the SSP about the auction result and the SSP serves the winner's ad to the audience.

\begin{figure}[!h]
\centering
\includegraphics[width=1.0\linewidth]{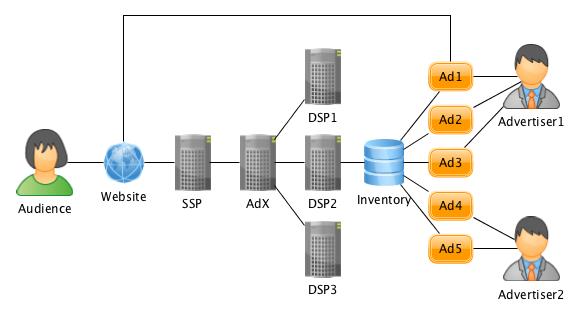}
\caption{Real Time Bidding}
\end{figure}

There are two popular payment modes for advertisers
\begin{enumerate}
\item \textbf{P4P}(Pay For Performance): the advertiser sets a CPP(Cost Per Performance)
    and pays DSP the CPP times the units of performance delivered by DSP(e.g. 1\$/click*10clicks=10\$).
\item \textbf{P4U}(Pay For Usage): the advertiser sets a CR(Commission Rate)
    and pays DSP the total bidding cost plus the fraction of it as commission(e.g. (1+10\%)*100\$=110\$).
\end{enumerate}

DSP is interested in optimizing one of the following objectives
\begin{enumerate}
\item \textbf{Revenue}: the total amount of money(e.g. 50\$) earned from advertisers through either payment mode mentioned above.
\item \textbf{Performance}: the total units of performance(e.g. 20 clicks) delivered to advertisers.
\end{enumerate}

During the optimization, several constraints must be satisfied
\begin{enumerate}
\item \textbf{Budget Upper Bound}: the maximum amount of money the advertiser is willing to spend in DSP for a certain period of time
  (e.g. 100\$/day).
\item \textbf{ROI Lower Bound}: the minimum value of ROI(Return On Investment) which is defined as,
  for DSP, the revenue earned from advertisers over the bidding cost payed to AdX
  (e.g. DSP ROI is 1.1 when DSP earns 110\$ and pays 100\$) and,
  for advertiser, the performance delivered by DSP over the money spent in DSP
  (e.g. advertiser ROI is 0.8 when advertiser spends 100\$ for 80 clicks).
\end{enumerate}

It's essential for DSP to find an optimal ad selection and bid price determination strategy
    which maximizes revenue or performance under budget and ROI constraints in P4P or P4U mode.
We solve this problem by

\begin{enumerate}
\item formalizing the DSP problem as a constrained optimization problem(Section \ref{Formalization}),
\item proposing the augmented MMKP(Multi-choice Multi-dim-ensional Knapsack Problem) with general solution(Section \ref{AugmentedMMKP}),
\item and demonstrating the DSP problem is a special case of the augmented MMKP and deriving specialized strategy(Section \ref{Solution}).
\end{enumerate}

Our strategy is verified through simulation(Section \ref{Simulation}) and
    outperforms state-of-the-art strategies in real application(Section \ref{Application}).
To the best of our knowledge, our solution is the first dual based DSP bidding framework
    that is derived from strict second price auction assumption and
    generally applicable to the multiple ads scenario with various objectives and constraints.
These are the main contributions of this document.

Before further discussion, it's worth to mention several points about our problem configuration.
First, PPI(Performance Per Impression) is defined as the expected performance of one impression with certain ad
    and its accurate prediction is of great importance in performance estimation.
However, PPI prediction is beyond the scope of this document
    and we assume that the PPI is always explicitly provided in the rest of our discussion.
Second, it is assumed that all advertisers agree to the same payment mode and performance metric
    and DSP prefers to optimize a pure objective rather than a hybrid one.
Third, the CPP in P4P mode or CR in P4U mode are set on the ad level,
    i.e. the advertiser is able to set different CPP or CR for his ads.
And the constraints are set on ad group level,
    e.g. the budget might be shared by ads of the same advertiser
    and DSP might be interested in controlling its global ROI.
At last, the ROI lower bound for advertiser could also be interpreted as the CPP upper bound
    which might be more familiar to some readers.

\section{Related Works}

\cite{M6D} suggests a linear bidding strategy which, given base price,
    bids in proportion to the relative quality of impression.
However, their method is a heuristic one and lacks theoretical foundations.

Based on calculus of variations, \cite{WeinanZhang2014} suggests a non-linear relationship between optimal bid price and KPIs.
However, their strategy is derived from first price auction assumption which doesn't hold in RTB.
Besides, winning rate is explicitly modeled as a function of bid price in \cite{WeinanZhang2014}.
To find the analytical solution of the optimal bid price,
    the winning rate function must be of specific forms,
    which makes their method inflexible.

Both win rate and winning price are estimated in \cite{XiangLi2014}, and the corresponding bidding strategy is provided.
However, their strategy doesn't consider any constraints(i.e. budget) which are common in real DSP applications.

While all above researches consider only one campaign,
    \cite{WeinanZhang2015} extends \cite{WeinanZhang2014} and proposes bidding strategy for multiple campaigns.
However, \cite{WeinanZhang2015} also shares the drawbacks of \cite{WeinanZhang2014} as listed above.

\cite{Joint2016} studies the joint optimization of multiple objectives with priorities.
\cite{Lift2016} argues that the bid price should be decided
    based on the performance lift rather than absolute performance value.
Risk management of RTB is discussed and risk-aware bidding strategy is proposed in \cite{Risk2017}.
By modeling the state transition of auction competition,
    the optimal bidding policy is derived in \cite{Reinforce2017} based on reinforcement learning theory.

The probability estimation of interested feedbacks plays a central role in performance based advertising.
CTR(Click Through Rate) prediction is of great importance and extensively studied by researchers.
FTRL-Proximal, an online learning algorithm, is proposed in \cite{Google2013} and sparse model is learned for CTR prediction.
In \cite{Facebook2014}, a hybrid model which combines decision trees with logistic regression
    outperforms either of these methods on their own.
In \cite{FFM2016}, field-aware factorization machines are used to predict CTR.
Compared with clicks, the conversions are even more rare and harder to predict.
To tackle the data sparseness, a hierarchical method is proposed in \cite{CVR} for CVR(Conversion Rate) prediction.
Feedbacks are usually delayed in practice and \cite{DelayedFeedback} tries to
    distinguish negative training samples without feedbacks eventually from those with delayed ones.

Bidding landscape is studied in \cite{YingCui2011} and log normal is used to model the distribution of winning price.
\cite{Wu2015} predicts win price with censored data, which utilizes both winning and losing samples in the sealed auction.
Traffic prediction for DSP is discussed in \cite{Traffic2016}.
Budget pacing is achieved through throttling in \cite{Throttle2015} and bid price adjustment in \cite{Pacing2013}.

Our work is mainly inspired by \cite{YeChen2011} in which compact allocation strategy,
    after modeling its problem as linear programming, is derived from complementary slackness.
Sealed second price auction is studied in \cite{SSPA1961}.
After all, DSP problem is a sort of online matching problem and \cite{Mehta} is an informative survey of this area.

\section{Formalization} \label{Formalization}

\subsection{Primal}

The DSP problem could be formalized as follows.
Once we bid $Impression$ with $Ad$, it results in gain $V$ and resource consumptions $W$, both of which are functions of $BidPrice$.
Our total gain should be maximized under resource constraints $B$ with $x$ and $BidPrice$ as variables.
In addition, each $Impression$ should be distributed to no more than one $Ad$.
To conquer the computational hardness, indicator variable $x$ is relaxed from $\{0, 1\}$ to $[0, 1]$.
Although most kinds of resources(e.g. budget) are sort of private and only accessible to very limited number of $Ad$s in practice,
    we assume, without loss of generality, that all resources are public and shared by all $Ad$s in this formalization.

\begin{alignat}{2}
    \max\limits_{\sx, \sbp} \quad & \sumij \sx \sV(\sbp) \quad    & {} \\
    \mbox{s.t.} \quad             & \dspresourceconstraint \quad  & \forall k \\
    \quad                         & \assignmentconstraint \quad   & \forall i \\
    \quad                         & \sx \ge 0 \quad               & \forall i,j
\end{alignat}

$i \inRange{N}$ is the index of $Impression$

$j \inRange{M}$ is the index of $Ad$

$k \inRange{K}$ is the index of $Constraint$

$\sx$ is a relaxed variable, indicating whether $Impression_i$ should be given to $Ad_j$

$\sbp$, short for $BidPrice_{ij}$, is a variable

$\sV(bp)$ is the gain function of $BidPrice$ with support $[0, \infty)$

$\sW(bp)$ is the $k$-th resource consumption function of $BidPrice$ with support $[0, \infty)$

$\sB$ is a resource limit constant

The above formalization might seem too abstract to capture the details of
    those practical objectives and constraints discussed in Section \ref{Introduction}.
To make things clearer, we
\begin{enumerate}
\item derive the expected winning probability and bidding cost under second price auction assumption(Section \ref{SecondPriceAuction}),
\item define the utility function family $\uff$ based on previous derivation(Section \ref{UtilityFunctionFamily}),
\item and show how to systematically encode those practical objectives and constraints into above formalization
    by setting $B$ and choosing $V(bp)$ and $W(bp)$ from $\uff$(Section \ref{ObjectivesAndConstraints}).
\end{enumerate}

\subsection{Second Price Auction} \label{SecondPriceAuction}

Most AdXes adopt sealed second price auction mechanism in which
    the DSP with the highest bid price wins and pays the second highest bid price.
For example, three DSPs bid 2\$, 1\$, 3\$ respectively, so the third DSP wins and pays 2\$. 
Furthermore, only the winner has access to the second highest bid price
    while the others observe nothing except the fact that they lose.

Due to the dynamic nature of auction, the outcome is random.
To model this uncertainty, $p_i(x)$ is defined as
    the distribution of the highest bid price among all other DSPs' bid prices for $Impression_i$ with support $[0, \infty)$.
In another word, the most competitive DSP will bid $x$ for $Impression_i$ with probability $p_i(x) \mathrm{d} x$.

To win $Impression_i$, our $BidPrice$ must be higher than $x$, but we will only pay $x$ eventually.
Then the expected winning probability and bidding cost for our DSP could be defined as follows.
As our $BidPrice$ goes infinite, we'll win $Impression_i$ with probability 1 and our bidding cost must be the mean of $p_i(x)$.

\begin{definition}
$Prob_i(BidPrice)= \int_0^{BidPrice} p_i(x) \mathrm{d} x$
\end{definition}

\begin{definition}
$Cost_i(BidPrice)= \int_0^{BidPrice} x p_i(x) \mathrm{d} x$
\end{definition}

It is the non-negativity property of $p_i(x)$ and the integral forms of $Prob_i(BidPrice)$ and $Cost_i(BidPrice)$
    which play a central role in our theory(Section \ref{UtilityFunctionFamily}).
Except that, we make little, if any, assumption about the distribution family of $p_i(x)$.
In addition, in some special cases, even explicit modeling of $p_i(x)$ is unnecessary
    (Section \ref{DSPDualBasedStrategy} \& \ref{DSPNumericOptimization}), which simplifies the implementation of our strategy.

Whenever it's mandatory, $p_i(x)$ could be modeled with method proposed by \cite{Wu2015}.
We could pick a distribution family $p(x; \theta)$(e.g. log normal) with parameter $\theta$ and
    learn a parameter predictor $\hat\theta(i)$ from historical bidding data which maximizes the following likelihood.

\begin{alignat}{1}
\prod\limits_{i \in Win} p_i(Cost_i; \hat\theta(i)) \prod\limits_{i \in Lose} \int_{BidPrice_i}^{\infty} p_i(x; \hat\theta(i)) \mathrm{d} x
\end{alignat}

The likelihood could be separated into two parts, i.e. one for the impressions we won and the other for those we lost.
For any $Impression_i$ that we won, the bid price of the most competitive DSP must be equal to our $Cost_i$, which suggests the first part.
Otherwise, the only thing for sure is that it must be higher than our $BidPrice_i$, which suggests the second part.

\subsection{Utility Function Family} \label{UtilityFunctionFamily}

The practical $V(bp)$ and $W(bp)$ in DSP problem come from a certain family
    which is defined here and whose properties are shown without proof.

\begin{definition}
$\uff$ is the function family that $\forall f \in \uff$ is of the form
    $ \uf = \pprob Prob(bp) + \pcost Cost(bp) = \int_0^{bp} (\pprob + \pcost x)p(x)dx $.
\end{definition}

\begin{theorem} \label{DerivationTheorem}
Given $\forall f \in \uff$, we have $f'=(\pprob + \pcost{}bp)p(bp)$.
\end{theorem}

\begin{theorem} \label{ArgMaxTheorem}
Given $\forall f \in \uff$, we have $\argmax\limits_{bp} f = - \pprob / \pcost$.
\end{theorem}

\begin{theorem} \label{ComparisonTheorem}
Given $\forall g,h \in \uff$ with shared $p(x)$ and $\pcost$, we have $\max\limits_{bp} g \ge \max\limits_{bp} f$
    if and only if $\pprob_g \ge \pprob_f$.
\end{theorem}

\begin{theorem} \label{SecondDerivationTheorem}
Given $\forall g,h \in \uff$ with shared $p(x)$,
    we have $\frac{d^2h}{dg^2} = \frac{\pprob_g \pcost_h - \pcost_g \pprob_h}{(\pprob_g + \pcost_g bp)^3 p(bp)}$.
\end{theorem}

All above theorems are listed here for summarization purpose and will be referenced when actually used.
It's safe to skip them for now and come back later.

\subsection{Objectives and Constraints} \label{ObjectivesAndConstraints}

There are 4 practical objectives as listed in Table \ref{TableObjectives},
    i.e. revenue and performance objectives in P4P and P4U modes.
It's straightforward to encode those objectives into standard form by definition.

\begin{table*}
\caption{Practical Objectives\label{TableObjectives}}
\begin{center}
\begin{tabular}{|c|c|c|c|c|}
\hline
\mr{2}{Mode}   & \mr{2}{Type}       & \mr{2}{Definition}  & \mc{2}{$\sV(bp)$} \\
\cline{4-5}
               &                    &                     & $\pprob$ & $\pcost$ \\
\hline
\mr{2}{P4P}    & Revenue            & $\sRevenuePforP$    & $\sCPI$  & 0 \\
\cline{2-5}
               & Performance        & $\sPerformance$     & $\sPPI$  & 0 \\
\hline
\mr{2}{P4U}    & Revenue            & $\sRevenuePforU$    & 0        & $1+\sCR$ \\
\cline{2-5}
               & Performance        & $\sPerformance$     & $\sPPI$  & 0 \\
\hline
\end{tabular}
\end{center}
\end{table*}

There are 6 practical constraints as listed in Table \ref{TableConstraints},
    i.e. budget, DSP ROI and advertiser ROI constraints in P4P and P4U modes.
Constraints like budget could be expressed in standard form naturally.
Others, though not so obvious at the first glance, could be rewritten into standard form as well.

Take DSP ROI constraint in P4P mode for example.
By definition, we have
\begin{alignat}{1}
\frac{\sRevenuePforP}{\sBiddingCost}\ge\sROI
\end{alignat}
After multiplying both sides with the denominator,
    subtracting both sides with the nominator,
    and combining items by $\sx$, we have
\begin{alignat}{1}
\sumijx{\{\sROI\sCost-\sCPI\sProb\}}\le0
\end{alignat}
It's easy to encode this constraint into standard form with $\pprob=-\sCPI$, $\pcost=\sROI$ and $\sB=0$.

\begin{table*}
\caption{Practical Constraints\label{TableConstraints}}
\begin{center}
\begin{tabular}{|c|c|c|c|c|c|}
\hline
\mr{2}{Mode} & \mr{2}{Type}   & \mr{2}{Definition}                             & \mc{2}{$\sW(bp)$}   & \mr{2}{$\sB$} \\
\cline{4-5}
             &                &                                                & $\pprob$            & $\pcost$         & \\
\hline
\mr{3}{P4P}  & Budget         & $\sRevenuePforP \le \sBudget$                  & $\sCPI$             & 0                & $\sBudget$ \\
\cline{2-6}
             & DSP ROI        & $\frac{\sRevenuePforP}{\sBiddingCost}\ge\sROI$ & $-\sCPI$            & $\sROI$          & 0 \\
\cline{2-6}
             & Advertiser ROI & $\frac{\sPerformance}{\sRevenuePforP}\ge\sROI$ & $\sCPI\sROI-\sPPI$  & 0                & 0 \\
\hline
\mr{3}{P4U}  & Budget         & $\sRevenuePforU \le \sBudget$                  & 0                   & $(1+\sCR)$       & $\sBudget$ \\
\cline{2-6}
             & DSP ROI        & $\frac{\sRevenuePforU}{\sBiddingCost}\ge\sROI$ & 0                   & $\sROI-(1+\sCR)$ & 0 \\
\cline{2-6}
             & Advertiser ROI & $\frac{\sPerformance}{\sRevenuePforU}\ge\sROI$ & $-\sPPI$            & $\sROI(1+\sCR)$  & 0 \\
\hline
\end{tabular}
\end{center}
\end{table*}

\section{Augmented MMKP} \label{AugmentedMMKP}

\subsection{Primal}

Now we propose the augmented MMKP which could be formalized as follows and seen as an extension of MMKP with infinitely many sub-choices.
In the original MMKP, both $V$ and $W$ are constants, while, in the augmented MMKP, $V$ is variable and $W$ becomes function of $V$.
Our main-choice and sub-choice are indicated by $x$ and corresponding $V$ respectively.

\begin{alignat}{2}
    \max\limits_{\sx, \sV} \quad & \sumij \sx \sV \quad              & {} \\
    \mbox{s.t.} \quad            & \ammkpresourceconstraint \quad    & \forall k \\
    \quad                        & \assignmentconstraint \quad       & \forall i \\
    \quad                        & \sx \ge 0 \quad                   & \forall i,j
\end{alignat}

$i \inRange{N}$ is the index of $Item$

$j \inRange{M}$ is the index of $User$

$k \inRange{K}$ is the index of $Constraint$

$\sx$ is a relaxed variable, indicating whether $Item_i$ should be given to $User_j$

$\sV$ is a gain variable

$\sW(V)$ is the $k$-th resource consumption function of $V$ with support $[0, \infty)$

$\sB$ is a resource limit constant

\subsection{Dual}

We define several basic functions and show the dual of augmented MMKP based on them.

\begin{definition}
$\sF(V; \valpha) = V - \sumk \salpha \sW(V)$
\end{definition}

\begin{definition}
$\sV(\valpha) = \argmax\limits_V \sF(V; \valpha)$
\end{definition}

\begin{definition}
$\sS(\valpha) = \max\limits_V \sF(V; \valpha)$
\end{definition}

\begin{alignat}{2}
    \min\limits_{\salpha, \sbeta} \quad & \sumk \salpha \sB + \sumi \sbeta \quad   & {} \\
    \mbox{s.t.} \quad                   & \scoreconstraint \quad                   & \forall i,j \\
    \quad                               & \salpha \ge 0 \quad                      & \forall k \\
    \quad                               & \sbeta \ge 0 \quad                       & \forall i
\end{alignat}

$\sF(V; \valpha)$ serves as sort of score function which is used to estimate
    the utility of distributing $Item_i$ to $User_j$ with sub-choice $V$.
It could be interpreted as the compromised gain function in which
    our original gain $V$ is degenerated by resource consumption $W$ with opportunity price $\alpha$.

\subsection{Strong Duality}

The strong duality of augmented MMKP is provable under mild assumption.
A brief proof is provided here and more details are revealed in the appendix.

\begin{theorem} \label{StrongDualityTheorem}
If $W(V)$ is convex function of $V$, strong duality of augmented MMKP holds.
\end{theorem}

\begin{proof}
Several auxiliary problems are defined in Table \ref{TableAuxiliaryProblems}.
P is the primal and could be separated into 2 nested steps.
The inner step, given $\sx$, maximizes objective with $\sV$ as variables.
The outer step, maximizes objective with $\sx$ as variables.

\begin{table}
\caption{Auxiliary Problems\label{TableAuxiliaryProblems}}
\begin{center}
\begin{tabular}{|c|c|}
\hline
Name   & Description \\
\hline
P      & outer[inner] \\
\hline
D      & dualize(P) \\
\hline
DD     & dualize(dualize(P)) \\
\hline
d      & outer[dualize(inner)] \\
\hline
dd     & outer[dualize(dualize(inner))] \\
\hline
\end{tabular}
\end{center}
\end{table}

Since $W(V)$ is convex function of $V$, inner is a strong duality problem, inner = dualize(inner), P = d.
Since dualize(inner) is a strong duality problem, dualize(inner) = dualize(dualize(inner)), d = dd.
Since D is a strong duality problem, D = DD.
dd and DD happen to have the same form, dd = DD.
As a result, P = D, strong duality of augmented MMKP holds.
\end{proof}

\subsection{Dual Based Strategy}

With strong duality satisfied, several important properties
    are claimed about the optimal solution of both primal and dual problems(i.e. $\sx^*$, $\sV^*$, $\valpha^*$ and $\sbeta^*$),
    based on which we propose the dual based strategy.

\begin{theorem}
$\sV^* = \sV(\valpha^*)$.
\end{theorem}

\begin{theorem}
$\sx^*(\sbeta^* - \sS(\valpha^*)) = 0$.
\end{theorem}

\begin{corollary}
If $\sS(\valpha^*) < 0$, we have $\sx^* = 0$.
\end{corollary}

\begin{proof}
Since $\sS(\valpha^*) < 0$ and $\sbeta^* \ge 0$, we have $\sbeta^* > \sS(\valpha^*)$.
Now that $\sbeta^* - \sS(\valpha^*) > 0$ and $\sx^* \ge 0$,
    taking above theorem into consideration, $\sx^*$ must be 0,
    that is $Item_i$ should not be distributed to $User_j$.
\end{proof}

\begin{corollary}
If $S_{ij_1}(\valpha^*) < S_{ij_2}(\valpha^*)$, we have $x_{ij_1}^* = 0$.
\end{corollary}

\begin{proof}
Similarly, since $S_{ij_1}(\valpha^*) < S_{ij_2}(\valpha^*)$ and $\sbeta^* \ge S_{ij_2}(\valpha^*)$,
    we have $\sbeta^* > S_{ij_1}(\valpha^*)$.
Now that $\sbeta^* - S_{ij_1}(\valpha^*) > 0$ and $x_{ij_1}^* \ge 0$, 
    taking above theorem into consideration, $x_{ij_1}^*$ must be 0,
    that is $Item_i$ should not be distributed to dominated $User_{j_1}$.
\end{proof}

\begin{theorem}
$\sbeta^*(\sum\limits_j \sx^* - 1) = 0$.
\end{theorem}

\begin{corollary}
If $\exists j$ that $\sS(\valpha^*) > 0$, we have $\sum\limits_j \sx^* = 1$.
\end{corollary}

\begin{proof}
Since $\sS(\valpha^*) > 0$ and $\sbeta^* \ge \sS(\valpha^*)$, we have $\sbeta^* > 0$.
Now that $\sbeta^* > 0$ and $\sum\limits_j \sx^* - 1 \le 0$,
    taking above theorem into consideration, $\sum\limits_j \sx^*$ must be 1,
    which means $Item_i$ should not be discarded.
\end{proof}

In summary, for each $Item_i$, every $User_j$ should propose its own best score $\sS^*$ achieved by $\sV^*$.
$Item_i$ should be awarded to the dominating $User_{j^*}$ if its best score $S_{ij^*}^*$ is positive and discarded if that is negative.
Theoretically speaking, while most of which are determined by above corollaries, behaviors remain undefined in two special cases.
First, there might be multiple dominating users with the same best score.
Second, the best score of dominating user might be exactly zero.
In practice, however, both cases are probably rare due to the high resolution of items and users, and prone to cause relatively limited damage.
Ties could be broken by random or heuristics.

\begin{algorithm}
\caption{Dual Based Strategy for Augmented MMKP}

\For{$Item_i \in \mathscr{I}$}
{
  \For{$User_j \in \mathscr{U}$}
  {
    $\sF(V; \valpha^*) = V - \sumk \salpha \sW(V)$

    $\sV^* = \sV(\valpha^*) = \argmax\limits_V \sF(V; \valpha^*)$

    $\sS^* = \sS(\valpha^*) = \max\limits_V \sF(V; \valpha^*)$
  }
  $j^* = \argmax\limits_j \sS^*$
  
  \If{$S_{ij^*}^* \ge 0$} { bid with ($User_{j^*}$, $V_{ij^*}^*$) }
}
\end{algorithm}

\subsection{Numeric Optimization}

Note that, during the execution of the dual base strategy,
    only the $\valpha^*$ is mandatory while the others(i.e. $\sx^*$, $\sV^*$ and $\sbeta^*$) could be recovered with $\valpha^*$,
    which makes our strategy storage efficient.
Next, we propose the numeric method to solve $\valpha^*$.

\begin{definition}
$\sbeta(\valpha) = \max \{ 0, \sS(\valpha) \forall j \}$
\end{definition}

\begin{definition}
$\sG(\valpha) = \sum\limits_k \frac{\salpha \sB}{N} + \sbeta(\valpha)$
\end{definition}

By fixing $\valpha$ in the dual problem, $\sbeta^*$ could be calculated as $\sbeta^* = \sbeta(\valpha)$.
Then the dual problem could be rewritten as 
\begin{alignat}{2}
\min\limits_{\valpha \ge 0} & \sum\limits_i \sG(\valpha)
\end{alignat}
    and solved by SGD(Stochastic Gradient Descent).
Due to the convexity of dual problem, it must converge to the global optimal $\valpha^*$.

\section{Solution} \label{Solution}

\subsection{Dual}

We define corresponding basic functions and show the dual of DSP problem based on them.

\begin{definition}
$\sF(bp; \valpha) = \sV(bp) - \sumk \salpha \sW(bp)$
\end{definition}

\begin{definition}
$\sbp(\valpha) = \argmax\limits_{bp} \sF(bp; \valpha)$
\end{definition}

\begin{definition}
$\sS(\valpha) = \max\limits_{bp} \sF(bp; \valpha)$
\end{definition}

\begin{alignat}{2}
    \min\limits_{\salpha, \sbeta} \quad & \sumk \salpha \sB + \sumi \sbeta \quad & {} \\
    \mbox{s.t.} \quad                   & \scoreconstraint \quad                 & \forall i,j \\
    \quad                               & \salpha \ge 0 \quad                    & \forall k \\
    \quad                               & \sbeta \ge 0 \quad                     & \forall i
\end{alignat}

Note that, in DSP problem, our sub-choice is indicated by $bp$ rather than $V$.
Since $\sF$ is the linear combination of $\sV(bp)$ and $\sW(bp)$ from $\uff$ with shared $p_i(x)$,
    it must belong to $\uff$ too with its $\pprob_{\sF}$ and $\pcost_{\sF}$ as follows.
\begin{alignat}{2}
\pprob_{\sF}= & \pprob_{\sV} - \sum\limits_k \salpha \pprob_{\sW} \\
\pcost_{\sF}= & \pcost_{\sV} - \sum\limits_k \salpha \pcost_{\sW}
\end{alignat}
In practice, each ad is usually subjected to very limited number of constraints,
    which makes the calculation of $\pprob_{\sF}$ and $\pcost_{\sF}$ light-weighted.

\subsection{Strong Duality}

Due to the nice property of $\uff$, it's easy to check that,
    as to practical objectives and constraints(Section \ref{ObjectivesAndConstraints}),
    $W(bp)$ is indeed convex function of $V(bp)$,
    which immediately justifies the strong duality of DSP problem.

\begin{theorem} \label{DSPStrongDualityTheorem}
Strong duality of DSP problem holds.
\end{theorem}

\begin{proof}
According to Theorem \ref{DerivationTheorem}, $V'(bp)>0$ and $W(bp)$ must be function of $V(bp)$.
According to Theorem \ref{SecondDerivationTheorem}, $\frac{d^2W}{dV^2}\ge0$ and $W(bp)$ must be convex function of $V(bp)$.
As a result, according to Theorem \ref{StrongDualityTheorem}, strong duality of DSP problem holds.
\end{proof}

\subsection{Dual Based Strategy} \label{DSPDualBasedStrategy}

With strong duality satisfied, the dual based strategy developed for augmented MMKP is also applicable to DSP problem.
Generally speaking, $\sbp^*$ could be determined without $p_i(x)$ according to Theorem \ref{ArgMaxTheorem}.
In certain applications, $\pcost_{\sF}$ is the same for given $i$ and all $j$,
    then $j^*$ could also be determined independent of $p_i(x)$ according to Theorem \ref{ComparisonTheorem}.
By disposing of $p_i(x)$ completely from deciding process, it not only simplifies the computation,
    but also encourages $p_i(x)$ free training process.

\begin{algorithm}
\caption{Dual Based Strategy for DSP Problem}

\For{$Impression_i \in \mathscr{I}$}
{
  \For{$Ad_j \in \mathscr{A}$}
  {
    $\pprob_{\sF} = \pprob_{\sV} - \sum\limits_k \salpha \pprob_{\sW}$

    $\pcost_{\sF} = \pcost_{\sV} - \sum\limits_k \salpha \pcost_{\sW}$

    $\sF(bp; \valpha^*) = \int_0^{bp} (\pprob_{\sF}+\pcost_{\sF}x)p_i(x)dx$

    $\sbp^* = \sbp(\valpha^*) = \argmax\limits_{bp} \sF(bp; \valpha^*) = -\pprob_{\sF} / \pcost_{\sF}$

    $\sS^* = \sS(\valpha^*) = \max\limits_{bp} \sF(bp; \valpha^*)$
  }
  $j^* = \argmax\limits_j \sS^*$
  
  \If{$S_{ij^*}^* \ge 0$} { bid with ($Ad_{j^*}$, $bp_{ij^*}^*$) }
}
\end{algorithm}

\subsection{Numeric Optimization} \label{DSPNumericOptimization}

The numeric method developed for augmented MMKP is also applicable to DSP problem.
It's easy to prove that $\frac{d\sbeta(\valpha)}{d\salpha}$ must be
    either $-\sW(\sbp(\valpha))$ if the best score of the dominating $Ad_j$ is positive or $0$ otherwise.
This optimization method, though generally applicable, requires explicit modeling of $p_i(x)$.

Through executing our strategy in production environment, the randomized version of $\sW(\sbp(\valpha))$ is revealed
    and the gradients could be approximated with these feedbacks.
This optimization method is $p_i(x)$ free and much easier to implement.

\section{Simulation} \label{Simulation}

\subsection{Methodology}

To eliminate the uncertainty, our strategy is verified in P4P and P4U modes through simulation.
Due to the limited space, we focus on the P4P mode in the rest of Section \ref{Simulation}.

There are two simulation cases, i.e. one for revenue maximization and the other for performance maximization.
Two mocked ads $Ad_1$ and $Ad_2$ are created with $CPP_1=1$ and $CPP_2=2$.
Four mocked constraints are listed in Table \ref{TableConstraints}.
Budget of $Ad_1$ and $Ad_2$ are 20 and 10 respectively.
The global DSP ROI lower bound is 2, while the global advertiser ROI lower bound is 0.5.
As suggested by \cite{YingCui2011}, $p(x)$ is assumed to be log normal distribution $p(x;\mu,\sigma)$
    with mean $\mu$ and standard deviation $\sigma$ as parameters.
To mock the impressions, 200 tuples <$\mu_i$, $\sigma_i$, $PPI_{i1}$, $PPI_{i2}$> are drawn randomly.

\begin{table}
\caption{Mocked Constraints\label{TableConstraints}}
\begin{center}
\begin{tabular}{|c|c|c|c|}
\hline
k   & Type            & Parameter           & Scope   \\
\hline
1   & Budget          & $Budget^{(k)} = 20$ & $\{Ad_1\}$        \\
\hline
2   & Budget          & $Budget^{(k)} = 10$ & $\{Ad_2\}$        \\
\hline
3   & DSP ROI         & $ROI^{(k)} = 2$     & $\{Ad_1, Ad_2\}$  \\
\hline
4   & Advertiser ROI  & $ROI^{(k)} = 0.5$   & $\{Ad_1, Ad_2\}$  \\
\hline
\end{tabular}
\end{center}
\end{table}

Once the configurations are ready, $\valpha^*$ are solved by SGD(Section \ref{DSPNumericOptimization}).
After that, the dual based strategy is applied on the same cases and the consequent statistics are collected.

\subsection{Results and Analysis}

The statistics and $\valpha^*$ are listed in Table \ref{TableStatisticsAndAlpha}.
In both cases, all resources have non-negative surplus and no constraint is violated.
In addition, the gap between primal and dual objective values is negligible(Theorem \ref{DSPStrongDualityTheorem}).
As mentioned earlier, the $\alpha^*$ serves as so called opportunity price of the resource.
Intuitively speaking, waste of resource with positive surplus shouldn't lead to any opportunity cost.
As a result, the corresponding $\alpha^*$ tends to be 0.

\begin{table*}
\caption{Statistics and $\alpha^*$\label{TableStatisticsAndAlpha}}
\begin{center}
\begin{tabular}{|c|c|c|c|c|c|c|c|c|}
\hline
\mr{2}{Case}      & \mr{2}{Description}                          &\mr{2}{Primal} &\mr{2}{Dual}
                  & \mc{4}{Resource}                                        & \mr{2}{$\alpha^*$} \\
\cline{5-8}
                  &&&& $k$   & Limit       & Consumption     & Surplus      & \\
\hline
\mr{4}{$Case_1$}  & \mr{4}{\makecell{Revenue\\Maximization}}     & \mr{4}{2.164} &\mr{4}{2.164} 
                     & 1     & 20.000      & 0.865           & 19.135       & 0.000 \\
\cline{5-9}
                  &&&& 2     & 10.000      & 1.299           & 8.701        & 0.000 \\
\cline{5-9}
                  &&&& 3     & 0.000       & 0.000           & 0.000        & 4.749 \\
\cline{5-9}
                  &&&& 4     & 0.000       & -0.433          & 0.433        & 0.000 \\
\hline
\mr{4}{$Case_2$}  & \mr{4}{\makecell{Performance\\Maximization}} & \mr{4}{1.590} & \mr{4}{1.590}
                     & 1     & 20.000      & 1.100           & 18.900       & 0.000 \\
\cline{5-9}
                  &&&& 2     & 10.000      & 0.979           & 9.021        & 0.000 \\
\cline{5-9}
                  &&&& 3     & 0.000       & 2.747           & -2.747       & 3.654 \\
\cline{5-9}
                  &&&& 4     & 0.000       & -0.550          & 0.550        & 0.000 \\
\hline
\end{tabular}
\end{center}
\end{table*}

\section{Application} \label{Application}

\subsection{Scenario}

We also deploy our strategy in the DSP platform of Alibaba.
In our application, advertisers set budgets and pay for clicks, while DSP is willing to maximize revenue under daily global DSP ROI constraint.
There are so many ads in our inventory that it's impossible to go through each ad before auction deadline.
Although these budgets are quite large totally, they are relatively small on average.

With well calibrated CPP and PPI predictors, the problem could be transformed equivalently into one in P4P mode.
And to meet the latency requirement, the whole deciding process is decomposed into two stages with so called logical ad.

Logical ad should be seen as proxy of physical ads and binded with specific ad retrieval algorithm.
In the first stage, DSP is supposed to make decisions among just a few logical ads and respond in time.
In the second stage, once the chosen logical ad wins the auction, physical ad is lazily retrieved with corresponding algorithm.

\begin{figure}[!h]
\centering
\includegraphics[width=1.0\linewidth]{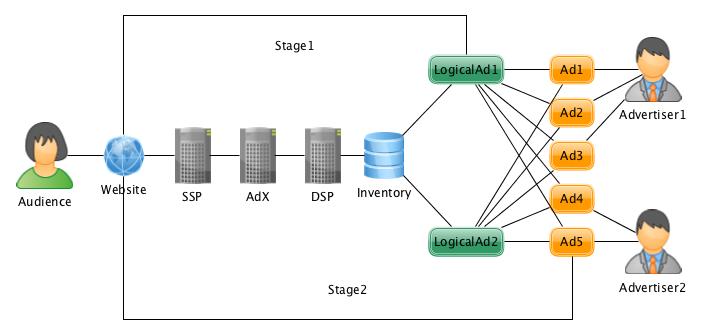}
\caption{Real Time Bidding with Logical Ad}
\end{figure}

Our logical ads are actually based on 4 heterogeneous ad retrieval algorithms whose details are beyond the scope of this document.
These algorithms are sorted by their historical performance in descending orders and 4 logical ads are constructed correspondingly.

In summary, our problem could be approximately modeled as, given 4 logical ads with literally unlimited budget,
    maximizing revenue under daily global DSP ROI constraint in P4P mode.
Since there is only one resource constraint, superscript $k$ is omitted and
    ROI is short for global DSP ROI in the rest of Section \ref{Application}.

According to our theory, we have $\pprob_{\sF}=(1+\alpha)\sCPI$ and $\pcost_{\sF}=-\alpha{}ROI$.
Since $\pcost_{\sF}$ is always $-\alpha{}ROI$, no $p_i(x)$ is required in deciding process as discussed in Section \ref{DSPDualBasedStrategy}.
To take full advantage of that, we adopt a simplified version of the $p_i(x)$ free optimization method
    suggested in Section \ref{DSPNumericOptimization}, i.e. $\dbiter$.

\subsection{Experiment Groups}

We compare our strategy with a variation of linear bidding strategy.
In \cite{M6D}, it's suggested that $\sbp=\frac{ActualCTR_{ij}}{CTR}Bid$ with $Bid$ set by operation team.
However, unlike $ActualCTR_{ij}$ which is independent of $\sbp$, $ActualROI_{ij}$ indeed varies with it.
As a result, we iteratively update $\sbp$ with $\liniter$.

We also apply optimal RTB theory to our application for comparison.
According to \cite{WeinanZhang2014}, we model the win probability as $w(bp;c)=\frac{bp}{c+bp}$ and bid with $\sbp=\ortbbp$,
    in which $c$ is fitted with method proposed by \cite{Wu2015} and $\lambda$ is iteratively tuned with $\ortbiter$.

Four experiment groups are shown in Table \ref{TableExperimentGroups}.
The first three groups are designed to compare different strategies with single logical ad,
    while the last group is used to test our strategy with multiple logical ads.

\begin{table*}
\caption{Experiment Groups\label{TableExperimentGroups}}
\begin{center}
\begin{tabular}{|c|c|c|c|c|c|}
\hline
Group    & Inventory                           & Strategy           & $\sbp$          & Iteration         & Period\\
\hline
$LIN$    & $\{LogicalAd_1\}$                   & Linear             & $\sbp$          & $\liniter$        & 24 hours \\
\hline
$ORTB$   & $\{LogicalAd_1\}$                   & Optimal RTB        & $\ortbbp$       & $\ortbiter$       & \mr{3}{10 minutes} \\
\cline{1-5}
$DB_{s}$ & $\{LogicalAd_1\}$                   & \mr{2}{Dual Based} & \mr{2}{$\dbbp$} & \mr{2}{$\dbiter$} & \\
\cline{1-2}
$DB_{m}$ & $\{LogicalAd_j|j \in \{1,2,3,4\}\}$ &                    &                 &                   & \\
\hline
\end{tabular}
\end{center}
\end{table*}

To eliminate potential bias, the experiment lasts for a whole ordinary week.
Bidding opportunities are distributed to each group randomly with equal probability.
For fairness, the same CPP and PPI predictors are shared by all groups.
The lower bound of daily ROI is set to 3.5.

Strategy parameters(i.e. $\sbp$, $\lambda$ and $\alpha$) are randomly initialized and
    periodically adjusted with actual ROI since last update.
The period is set to 24 hours for the $LIN$ group due to the data sparseness
    and 10 minutes as to the others for robustness and faster convergence.
Note that the more frequent update introduces inexplicitly a 10 minutes ROI constraint
    which is stricter than the daily one and might degenerate the theoretical optimal.

\subsection{Results and Analysis}

For each group, the daily statistics of four metrics are plotted in Figure \ref{Result},
    namely revenue, actual ROI, number of winning impressions and revenue per winning impression.

\begin{figure}[!h]
\centering
\includegraphics[width=1.0\linewidth]{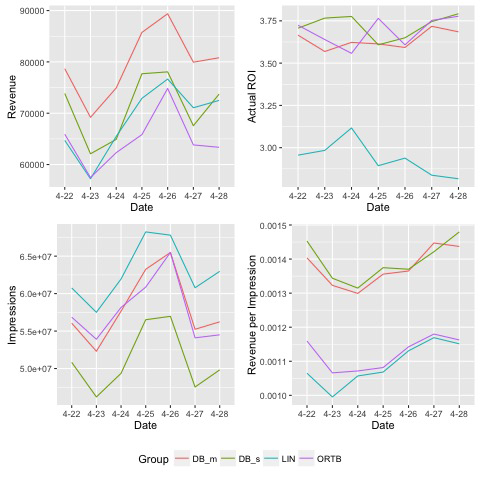}
\caption{Experiment Results\label{Result}}
\end{figure}

The $LIN$, though with theoretical optimal intact, tends to earn less revenue than the others in practice.
In addition, it usually violates the daily ROI constraint seriously, so it's an inferior strategy.

Compared with the $DB_s$ who claims a linear relationship between bid price and expected revenue,
    the $ORTB$, derived from first price auction assumption, suggests a non-linear one.
It is biased towards the impressions with low expected revenue and against those with high expected revenue,
    which leads to more impressions and lower averaged quality.
While the daily ROI constraint is satisfied by both strategies,
    the $DB_s$ earns more revenue than the $ORTB$. As a result, the $DB_s$ is superior theoretically and practically.

The $DB_m$, as an ensemble of four ad retrieval algorithms,
    achieves the most revenue without violation of the daily ROI constraint and becomes the best strategy.

\section{Conclusions and Future Works}

In this document, we propose a dual based DSP bidding strategy
    derived from second price auction assumption according to convex optimization theory.
Our strategy is verified through simulation and outperforms state-of-the-art strategies in real application.
It's a theoretically solid and practically effective strategy with simple implementation and various applications.

Three problems remain unsolved and deserve further study.
First, is there a better way to solve $\valpha^*$ of large scale in dynamic environment?
On the one hand, in a typical DSP, there will be millions of constraints shared by similar number of ads.
Each of the constraints deserves a $\alpha$, which makes the vector $\valpha$ very large.
On the other hand, billions of impressions are broadcast by AdX every day and bid by hundreds of DSPs simultaneously.
The bidding strategies are interactively adjusted by DSPs and the inventories are frequently updated by advertisers,
    which makes the bidding landscape unstable.
Both properties make the $\valpha^*$ hard to solve.

Second, how to construct and index logical ads automatically in massive ads applications, balancing latency and performance?
It's obvious that both deciding and training processes share the same ad evaluation and maximum determination style,
    which makes their computational complexities linearly related with the number of candidate logical ads.
At one extreme, each ad is represented by exactly one logical ad, and the consequent latency is unacceptable.
At the other extreme, all ads are represented by the only logical ad, while the performance might be seriously degenerated.
A proper compromise combined with efficient indexing tricks will accelerate both processes by orders of magnitude.

Third, how to optimally break ties when they are common and critical?
Take an imaginary scenario for example.
DSP is willing to maximize its revenue in P4P mode.
There are two identical ads with the same CPP and PPI,
    but they are targeted to overlapped sets of impressions and subjected to different budget constraints.
In this circumstance, resolution of impressions and ads is extremely low and ties are very prevalent.
To tackle the tie breaking problem, we might try randomized soft-max instead of hard-max during ad selection.
However, the theoretical soundness and practical effectiveness of this tie breaking strategy are to be verified.

\appendix

\section{Strong Duality Proof}

Here we give the detailed proof of the strong duality. We first prove that P $\le$ D by dualizing P.

\begin{flalign*}
    P = & - \min\limits_{\substack{\sx,\sV \\ \ammkpresourceconstraint \\ \assignmentconstraint \\ \sx \ge 0 }} \{ - \sumij \sx \sV \} \\
      = & - \min\limits_{\sx,\sV} \{ \max\limits_{\salpha,\sbeta,\sgamma \ge 0} \{ - \sumij \sx \sV \\
        & + \sumk \salpha [\sumij \sx \sW(\sV) - \sB] \\
        & + \sumi \sbeta (\sumj \sx - 1) + \sumij \sgamma(-\sx) \} \} \\
      = & - \min\limits_{\sx,\sV} \{ \max\limits_{\salpha,\sbeta,\sgamma \ge 0} \{ - \sumk \salpha \sB - \sumi \sbeta \\
        & + \sumij \sx [-\sV + \sumk \salpha \sW(\sV) + \sbeta - \sgamma] \} \} \\
    \le & - \max\limits_{\salpha,\sbeta,\sgamma \ge 0} \{ \min\limits_{\sx,\sV} \{ - \sumk \salpha \sB - \sumi \sbeta \\
        & + \sumij \sx [-\sV + \sumk \salpha \sW(\sV) + \sbeta - \sgamma] \} \} \\
      = & - \max\limits_{\substack{ \salpha,\sbeta \ge 0 \\ \scoreconstraint }} \{ -\sumk \salpha \sB - \sumi \sbeta \} \\
      = & \min\limits_{\substack{ \salpha,\sbeta \ge 0 \\ \scoreconstraint }} \{ \sumk \salpha \sB + \sumi \sbeta \} \\
      = & D &&
\end{flalign*}

Next, we prove that D = DD by dualizing DD.

\begin{flalign*}
    D = & \min\limits_{\salpha,\sbeta} \{ \max\limits_{\sx,\szeta,\seta \ge 0 } \{ \sumk \salpha \sB + \sumi \sbeta \\
        & + \sumij \sx[\sS(\valpha) - \sbeta] \\
        & + \sumk \szeta (-\salpha) + \sumi \seta (-\sbeta) \} \} \\
      = & \min\limits_{\salpha,\sbeta} \{ \max\limits_{\sx,\szeta,\seta \ge 0 } \{ \sumk \salpha (\sB - \szeta) \\
        & + \sumi \sbeta (1 - \sumj \sx - \seta) + \sumij \sx \sS(\valpha) \} \} \\
      = & \max\limits_{\sx,\szeta,\seta \ge 0 } \{ \min\limits_{\salpha,\sbeta} \{ \sumk \salpha (\sB - \szeta) \\
        & + \sumi \sbeta (1 - \sumj \sx - \seta) + \sumij \sx \sS(\valpha) \} \} \\
      = & \max\limits_{\substack{\sx,\szeta \ge 0 \\ \assignmentconstraint }} \{ \min\limits_{\salpha} \{
          \sumk \salpha (\sB - \szeta) + \sumij \sx \sS(\valpha) \} \} \\
      = & DD &&
\end{flalign*}

After that, we prove that P = d by dualizing the inner step of P with the outer step unchanged.

\begin{flalign*}
    P = & \max\limits_{\substack{ \sx \ge 0 \\ \assignmentconstraint }} \{
          \max\limits_{\substack{ \sV \\ \ammkpresourceconstraint }} \{
          \sumij \sx \sV \} \} \\
      = & \max\limits_{\substack{ \sx \ge 0 \\ \assignmentconstraint }} \{
          - \min\limits_{\substack{ \sV \\ \ammkpresourceconstraint }} \{
          - \sumij \sx \sV \} \} \\
      = & \max\limits_{\substack{ \sx \ge 0 \\ \assignmentconstraint }} \{
          - \min\limits_{\sV} \{ \max\limits_{\salpha \ge 0} \{
          - \sumij \sx \sV \\
        & + \sumk \salpha [\sumij \sx \sW(\sV) - B^{(k)}] \} \} \} \\
      = & \max\limits_{\substack{ \sx \ge 0 \\ \assignmentconstraint }} \{
          - \min\limits_{\sV} \{ \max\limits_{\salpha \ge 0} \{ - \sumk \salpha B^{(k)} \\
        & + \sumij \sx [-\sV + \sumk \salpha \sW(\sV)] \} \} \} \\
      = & \max\limits_{\substack{ \sx \ge 0 \\ \assignmentconstraint }} \{
          - \max\limits_{\salpha \ge 0} \{ \min\limits_{\sV} \{ - \sumk \salpha B^{(k)} \\
        & + \sumij \sx [-\sV + \sumk \salpha \sW(\sV)] \} \} \} \\
      = & \max\limits_{\substack{ \sx \ge 0 \\ \assignmentconstraint }} \{
          - \max\limits_{\salpha \ge 0} \{ - \sumk \salpha \sB - \sumij \sx \sS(\valpha) \} \} \\
      = & \max\limits_{\substack{ \sx \ge 0 \\ \assignmentconstraint }} \{
          \min\limits_{\salpha \ge 0} \{ \sumk \salpha \sB + \sumij \sx \sS(\valpha) \} \} \\
      = & d &&
\end{flalign*}

At last, we prove that d = dd by dualizing the inner step of d with the outer step unchanged.

\begin{flalign*}
    d = & \max\limits_{\substack{ \sx \ge 0 \\ \assignmentconstraint }} \{
          \min\limits_{\salpha} \{ \max\limits_{\szeta \ge 0} \{
          \sumk \salpha \sB + \sumij \sx \sS(\valpha) \\
        & + \sumk \szeta (-\salpha)\} \} \} \\
      = & \max\limits_{\substack{ \sx \ge 0 \\ \assignmentconstraint }} \{
          \min\limits_{\salpha} \{ \max\limits_{\szeta \ge 0} \{
          \sumk \salpha (\sB - \szeta) + \sumij \sx \sS(\valpha) \} \} \} \\
      = & \max\limits_{\substack{ \sx \ge 0 \\ \assignmentconstraint }} \{
          \max\limits_{\szeta \ge 0} \{ \min\limits_{\salpha} \{
          \sumk \salpha (\sB - \szeta) + \sumij \sx \sS(\valpha) \} \} \} \\
      = & \max\limits_{\substack{ \sx,\szeta \ge 0 \\ \assignmentconstraint }} \{
          \min\limits_{\salpha} \{
          \sumk \salpha (\sB - \szeta) + \sumij \sx \sS(\valpha) \} \} \\
      = & dd &&
\end{flalign*}

It's obvious that DD and dd are of the same form, so DD = dd. As a result, we have P = D and strong duality holds.

\bibliographystyle{ACM-Reference-Format}
\bibliography{DSP}

\end{document}